\documentclass[letterpaper, 10 pt, conference]{ieeeconf}  % Comment this line out if you need a4paper

\IEEEoverridecommandlockouts                              % This command is only needed if
\newcommand\moduleName[1]{\textsf{#1}\xspace}
\newcommand{\A}{\moduleName{A*}}
\newcommand{\wA}{\moduleName{wA*}}
\newcommand{\ARA}{\moduleName{ARA*}}
\newcommand{\ANA}{\moduleName{ANA*}}
\newcommand{\MHA}{\moduleName{MHA*}}
\newcommand{\AMHA}{\moduleName{A-MHA*}}
\newcommand{\MRA}{\moduleName{MRA*}}
\newcommand{\AMRA}{\moduleName{AMRA*}}
\newcommand{\RRT}{\moduleName{RRT*}}

\title{\LARGE \bf
\AMRA: Anytime Multi-Resolution Multi-Heuristic A*
}

\author{Dhruv Mauria Saxena, Tushar Kusnur, and Maxim Likhachev% <-this % stops a space
\thanks{The authors are with the Robotics Institute, Carnegie Mellon University, Pittsburgh, PA 15213, USA. {\small e-mail: \tt \{dsaxena, tkusnur, mlikhach\}@andrew.cmu.edu}. This work was sponsored by Mitsubishi Heavy Industries, Ltd.}%
}

\usepackage{xspace}
\usepackage{hyperref}

\usepackage{graphicx}
\usepackage{amsmath}
\usepackage{amssymb}

\usepackage{siunitx}
\usepackage[normalem]{ulem}

\usepackage{algorithm}
\usepackage[noend]{algpseudocode}
\let\oldReturn\Return
\renewcommand{\Return}{\State\oldReturn}

\usepackage{multirow}
\usepackage{booktabs}

\usepackage{geometry}
\usepackage[table,xcdraw,dvipsnames,x11names]{xcolor}
\definecolor{Red}{RGB}{255,0,0}
\definecolor{Green}{RGB}{0,128,0}
\definecolor{Blue}{RGB}{0,0,255}

\DeclareMathOperator*{\argmin}{arg\,min}

\newtheorem{theorem}{Theorem}

\usepackage{siunitx}
\usepackage{balance}

\begin{document}

\maketitle
\thispagestyle{empty}
\pagestyle{empty}

%%%%%%%%%%%%%%%%%%%%%%%%%%%%%%%%%%%%%%%%%%%%%%%%%%%%%%%%%%%%%%%%%%%%%%%%%%%%%%%%
\begin{abstract}

Heuristic search-based motion planning algorithms typically discretise the search space in order to solve the shortest path problem.
Their performance is closely related to this discretisation.
% Their computational complexity and theoretical properties of completeness and optimality are closely related to this discretisation.
A fine discretisation allows for better approximations of the continuous search space, but makes the search for a solution more computationally costly. A coarser resolution might allow the algorithms to find solutions quickly at the expense of quality.
For large state spaces, it can be beneficial to search for solutions across multiple resolutions even though defining the discretisations is challenging. The recently proposed algorithm Multi-Resolution A* (\MRA) searches over multiple resolutions. It traverses large areas of obstacle-free space and escapes local minima at a coarse resolution. It can also navigate so-called narrow passageways at a finer resolution. In this work, we develop \AMRA, an \textit{anytime} version of \MRA. \AMRA tries to find a solution quickly using the coarse resolution as much as possible. It then refines the solution by relying on the fine resolution to discover better paths that may not have been available at the coarse resolution. In addition to being anytime, \AMRA can also leverage information sharing between multiple heuristics. We prove that \AMRA is complete and optimal (in-the-limit of time) with respect to the finest resolution. We show its performance on 2D grid navigation and 4D kinodynamic planning problems.

\end{abstract}

%%%%%%%%%%%%%%%%%%%%%%%%%%%%%%%%%%%%%%%%%%%%%%%%%%%%%%%%%%%%%%%%%%%%%%%%%%%%%%%%

\section{Introduction}\label{sec:intro}

\begin{figure}[t]
    \centering
    \includegraphics[width=0.8\columnwidth]{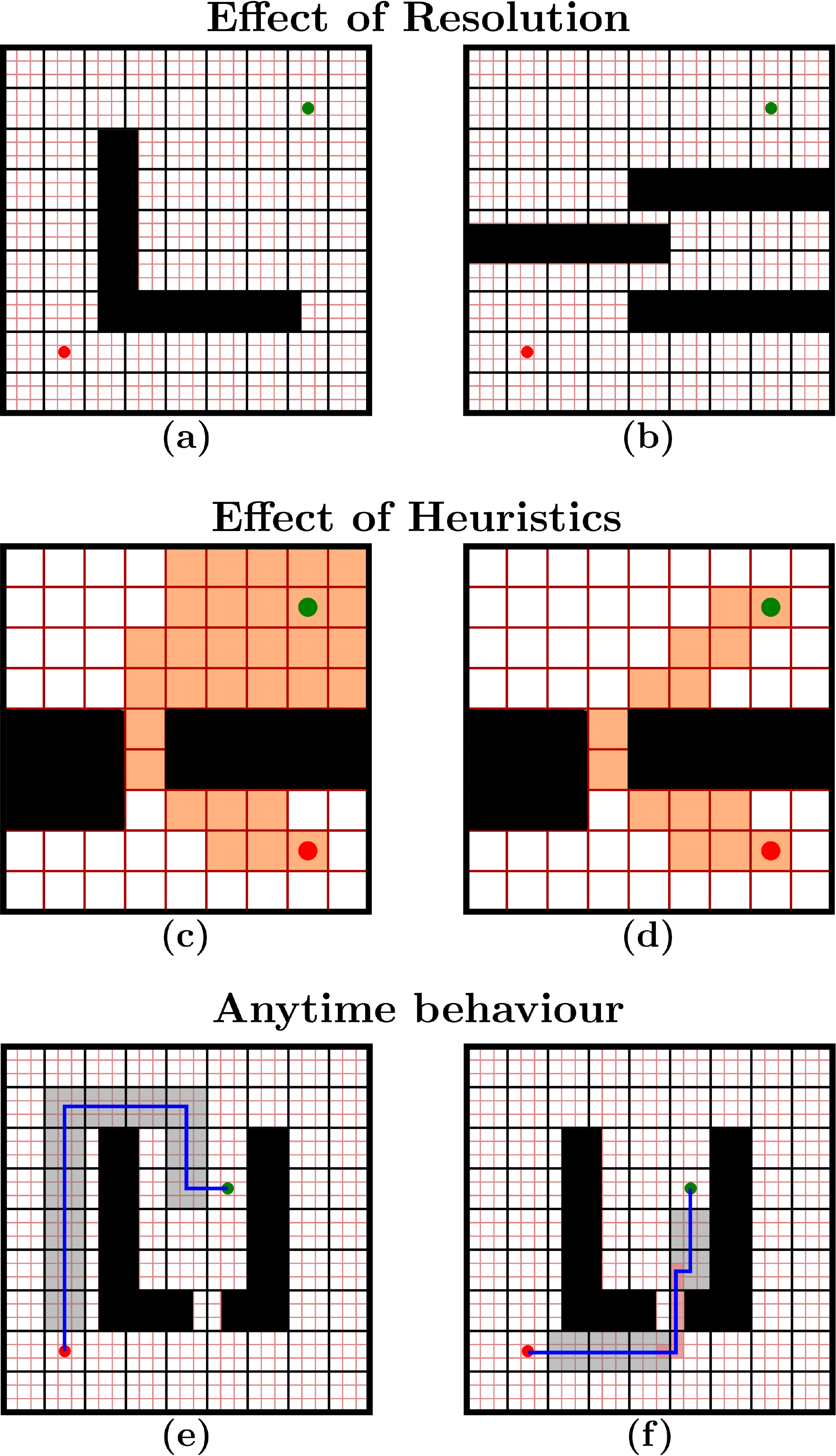}
    \caption{\textit{Effect of Resolution:} local minima can be explored quicker at coarser resolutions (a), while finer resolutions help navigate through narrow passageways (b); \textit{Effect of Heuristics:} a less informed heuristic (c) expands many more states than an informed heuristic (d); and \textit{Anytime behaviour}: an initial suboptimal solution (e) can be improved over time (f). The pictures show start states in \textcolor{Green}{green}, goal states in \textcolor{Red}{red}.}
    \label{fig:intro}
\end{figure}

Heuristic search algorithms for robot motion planning find least-cost solutions in discretised approximations of the continuous state space of the robot. They have shown impressive results in robot manipulation~\cite{CohenCL14}, navigation~\cite{LiuMAK18}, task planning~\cite{GarrettLK14}, and multi-robot coordination~\cite{WagnerC11}. The size of the search space for these algorithms is determined by the dimensionality of the robot state space and, crucially, the discretisation level of each of these dimensions~\cite{L2006}.
% Any solution found by these algorithms will consist of a sequence of discrete states, and thus the theoretical properties of completeness and optimality are tied to the level of discretisation used.
If the state space is discretised finely the search needs to explore a greater number of possible robot states in order to find a solution which is computationally costly. However at the same time, this higher resolution allows the search to find potential solutions through narrow passageways and dense obstacle clutter. A coarse discretisation of the state space is useful in relatively obstacle-free areas of the environment, and for the search to escape local minima where the heuristic estimate of the cost-to-goal is weakly correlated with the true cost-to-goal. The downside is that the search might fail to find a solution at that resolution.

At the same time, it is important for heuristic search algorithms to be instantiated with useful heuristics as they determine the computational effort spent exploring areas of the search space to find a solution. If a heuristic estimate of the cost-to-goal is poorly correlated with the true cost-to-goal, search algorithms can spend a lot of time expanding states in these local minima before finding a path to goal~\cite{Hoffmann01}. Multi-heuristic search algorithms~\cite{AineSNHL16} were developed to alleviate this problem by allowing search algorithms to be instantiated with multiple heuristics. These are not only easier to define for the practitioner, but also allow for information sharing between heuristics to better guide the search.

In this work we present \AMRA, an \textbf{A}nytime Multi-Heuristic \textbf{M}ulti-\textbf{R}esolution \textbf{\A} search algorithm that is capable of searching a state space at multiple levels of discretisation, share information between multiple heuristics, and improve the quality of the solution found over time. It is able to determine the appropriate resolution for exploring local minima and navigating across obstacle-free space and narrow passageways. It takes advantage of different heuristics being better correlated with the true cost-to-goal in different regions of the search space. Finally, with every iteration of the search loop, \AMRA is able to improve its solution with tighter suboptimality bounds and find the optimal solution in-the-limit of time.

Fig.~\ref{fig:intro} shows simple examples of three aspects of heuristic search that \AMRA encapsulates in one general algorithm while maintaining important theoretical properties of completeness and (sub-)optimality. First, in Fig.~\ref{fig:intro} (a), if we follow a greedy heuristic to the goal, the obstacle introduces a local minima with many more states at the fine resolution (light red grid) than the coarse resolution (black grid). In this case, running a search at coarse resolution will find a path to the goal with less computation. The downside of using only a coarse resolution is shown in Fig.~\ref{fig:intro} (b), where a solution only exists at the fine resolution. Fig.~\ref{fig:intro} (c-d) show the effect of using a less informed heuristic (Euclidean distance) vis-a-vis a perfect heuristic (backward Dijkstra search from the goal). For more complicated problems, different heuristics can be informative in different regions of the state space, and a search algorithm that can take advantage of this can greatly improve performance. Finally, an anytime algorithm like \AMRA relies heavily on the coarse resolution to quickly find an initial solution in Fig.~\ref{fig:intro} (e) (coarse states are highlighted in gray). It goes on to improve this solution over time to also include fine resolution states (highlighted in light red) in Fig.~\ref{fig:intro} (f).

%%%%%%%%%%%%%%%%%%%%%%%%%%%%%%%%%%%%%%%%%%%%%%%%%%%%%%%%%%%%%%%%%%%%%%%%%%%%%%%%

\section{Related Work}\label{sec:litreview}
\AMRA is an anytime, multi-heuristic, multi-resolution search algorithm for solving robot motion planning problems. It builds on the family of best-first search algorithms that traces its roots back to classic \A and Weighted A* (\wA) search algorithms~\cite{HartNR68,Pohl70}. For a large class of real-world robotics applications, optimal motion planning can be intractable due to the expansive nature of robot state spaces. Anytime algorithms allow us to solve problems in these domains by finding an initial highly suboptimal solution quickly, and spending any remaining planning budget to improve that solution. \ARA~\cite{LikhachevGT03} is an anytime version of \wA and provides bounds on solution suboptimality that Anytime \A~\cite{ZhouH02} does not. van den Berg et. al~\cite{BergSHG11} present \ANA, a non-paramateric version of \ARA\footnote{van den Berg et. al~\cite{BergSHG11} also contain a more thorough list of anytime \A algorithms.}.

While anytime algorithms have the ability to refine solutions over time, their performance is determined by the heuristic. The use of multiple heuristics within a search algorithm can dramatically improve search performance since different heuristics can offer better guidance in different regions of state space~\cite{Helmert06,AineSNHL16}. Recently, Natarajan et. al.~\cite{NatarajanSALC19} have also developed an anytime multi-heuristic algorithm.

Contemporaneous to the development of anytime and multi-heuristic algorithms, there has been work on developing algorithms that utilise multiple levels of discretisation of the robot state space. These multi-resolution algorithms rely on a coarser discretisation to navigate large regions of obstacle-free space, and revert to a finer discretisation to maneuver through narrow passageways~\cite{MooreA95,GarciaKB14,DuIL20}.

Most of the algorithms discussed above have provable bounds on solution suboptimality. Some sampling-based planners for robot motion planning~\cite{KaramanF11} offer a different notion of solution optimality. They are \textit{asymptotically} optimal, and thus will find the optimal solution given infinite time. As such, they can be interrupted early to exhibit an anytime behaviour.

We compare the performance of \AMRA in this paper with the three most closely related heuristic search algorithms (\ARA~\cite{LikhachevGT03}, \MRA~\cite{DuIL20} and \AMHA~\cite{NatarajanSALC19}) and against an asymptotically optimal sampling-based algorithm (\RRT~\cite{KaramanF11}).

%%%%%%%%%%%%%%%%%%%%%%%%%%%%%%%%%%%%%%%%%%%%%%%%%%%%%%%%%%%%%%%%%%%%%%%%%%%%%%%%

\section{Problem Formulation}\label{sec:problem}
We define a robot motion planning problem with the tuple $(\mathcal{X}, x_s, \mathcal{G})$, where $\mathcal{X}$ is the state space of the robot, $x_s \in \mathcal{X}$ is the start state, and $\mathcal{G} \subset \mathcal{X}$ is a set of goal states. $\mathcal{X}_{\text{free}} \subset \mathcal{X}$ denotes the obstacle-free space in the environment. A solution to the motion planning problem, if one exists, is a collision-free path from $x_s$ to $\mathcal{G}$.
% This implies that the robot remains in $\mathcal{X}_{\text{free}}$ along the solution path.

We assume access to a cost function $c: \mathcal{X} \times \mathcal{X} \rightarrow \mathbb{R}_{\geq 0}$ to compute the cost of an action between two robot states. The cost of a potential solution path $\pi = \{x_1, \ldots, x_n\}$ is denoted by overloading the definition of cost function $c$ as $c(\pi) = \sum_{i = 1}^{N-1} c(x_i, x_{i+1})$. Our goal in this paper is to solve the least-cost path planning problem and find the optimal path $\pi^* = \argmin_\pi c(\pi)$.

% In this work we solve motion planning problems on implicit graphs constructed over discretised approximations of $\mathcal{X}$. The graph $G = (V, E)$ is defined by a vertex set $V$, elements of which are states $x \in \mathcal{X}$, and an edge set $E$, elements of which are tuples $(x_i, x_j): x_i, x_j \in V$ if the robot can execute an action $a = (x_i, x_j)$ that takes it from $x_i$ to $x_j$. Specifically, we solve the least-cost path planning problem over $G$ for the path $\pi^* = \argmin_\pi \sum_{i=1}^{N-1} c(x_i, x_{i+1})$, where $c: E \rightarrow \mathbb{R}_+$ is the \textit{edge-cost} function.

%%%%%%%%%%%%%%%%%%%%%%%%%%%%%%%%%%%%%%%%%%%%%%%%%%%%%%%%%%%%%%%%%%%%%%%%%%%%%%%%

\section{Graph Construction and Search}\label{sec:graph}
We solve the least-cost robot motion planning problem with a heuristic search algorithm over a graph $G = (V, E)$. The vertex set $V \subset \mathcal{X}$ contains robot states. Edges $e = (x_i, x_j) \in E$ connect two vertices $x_i, x_j \in V$ if the robot can execute an action that takes it from $x_i$ to $x_j$. Thus each edge $e \in E$ is also an action $a$ in the robot action space $\mathcal{A}$.
% Heuristic search algorithms rely on three functions over $\mathcal{X}$. First, $g: V \rightarrow \mathbb{R}_{\geq 0}$ represents the \textit{cost-to-come} function such that $g(x)$ is the cost of the current best path from $x_s$ to $x$. Second, $h: V \rightarrow \mathbb{R}_{\geq 0}$ is the \textit{cost-to-goal} or \textit{heuristic} function - an estimate of the cost of the path between $x$ and $\mathcal{G}$. Third, $c: E \rightarrow \mathbb{R}_{\geq 0}$ is the \textit{edge cost} function and is used to calculate the cost of robot actions.

\subsection{Action Spaces}\label{sec:actions}

\AMRA constructs its vertex set $V = \bigcup_r V_r$ as a union over different levels of discretisation or resolutions $r$ of $\mathcal{X}$. Each vertex set $V_r$ has a corresponding edge set $E_r$ which make up the edges $E = \bigcup_r E_r$ used by \AMRA when constructing $G$. We represent each edge set $E_r$ with an action space $\mathcal{A}_r$ available to the robot. The core underlying assumption in this work and robot motion planning with multiple resolutions in general is that for every resolution $r$ being used, the robot has access to actions $\mathcal{A}_r$ that take it between two states $x_u, x_v \in V_r$. Note that this formulation allows for a state $x \in \mathcal{X}$ to exist at multiple resolutions $r$, and thus in multiple vertex and edge sets $V_r, E_r$.

\begin{figure}[t]
    \centering
    \includegraphics[width=0.25\columnwidth]{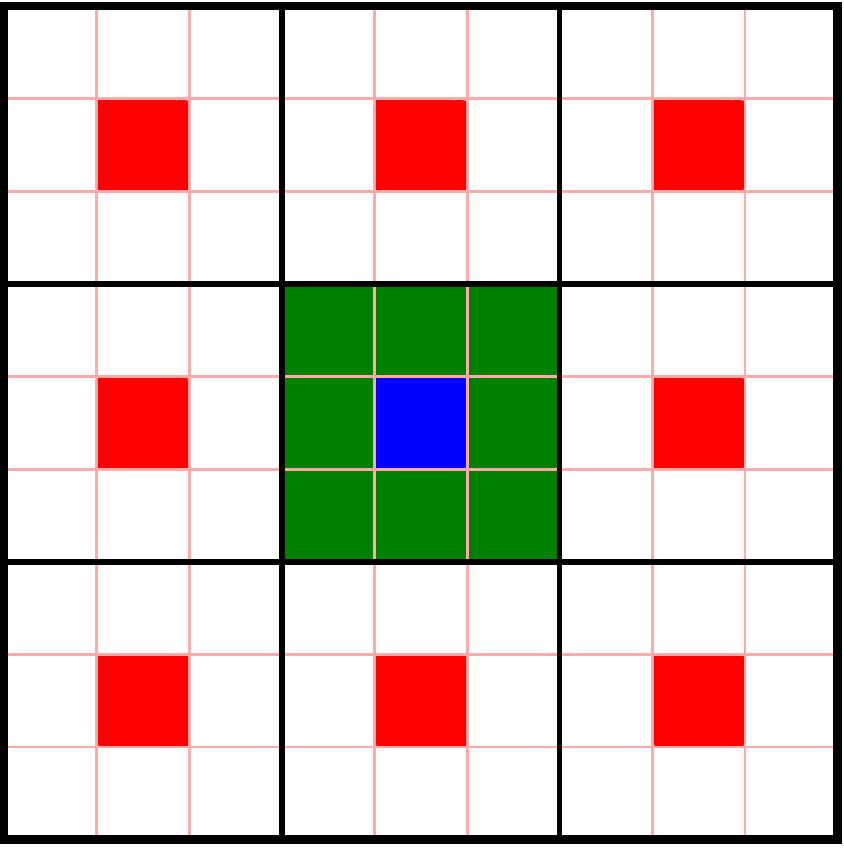}
    \caption{Multi-resolution action space for 8-connected grid navigation. The robot (at the \textcolor{Blue}{blue} state) can execute fine resolution actions to \textcolor{Green}{green} states, and coarse resolution actions to \textcolor{Red}{red} states.}
    \label{fig:actions}
\end{figure}

Fig.~\ref{fig:actions} shows an example of a multi-resolution action space for 2D grid navigation. \AMRA does not require that coarse actions be made up of fine actions, nor does it require any action to end in states that exist at multiple resolutions. However as we discuss in Section~\ref{sec:heuristics}, especially for multi-heuristic search, it can be useful to construct action spaces that lead to significant overlap between vertex sets at different resolutions.

\subsection{Multi-Heuristic Search}\label{sec:heuristics}
A heuristic function $h: V \rightarrow \mathbb{R}_{\geq 0}$ is an estimate of the \textit{cost-to-goal} from a state $x \in V$ on the graph $G$. Heuristics are \textit{admissible} if they under-estimate the true cost-to-goal from $x$ to $\mathcal{G}$ on $G$. \AMRA executes a multi-heuristic search derived from \MHA~\cite{AineSNHL16} which allows the use of multiple inadmissible heuristics. \AMRA parameterises heuristics by the resolution $r$ for which they are applicable. The search is initialised with a set of heuristics, at least one per resolution $r$ used. The \MHA framework allows us to use any number of additional heuristics for each resolution.

As with \MHA and \MRA, it is necessary to initialise \AMRA with an \textit{anchor} heuristic which is \textit{consistent}, i.e. a heuristic function $h$ such that $h(x_i) \leq h(x_j) + c(x_i, x_j) \forall e = (x_i, x_j) \in E$. We reserve the resolution $r = 0$ to refer to this anchor search. The anchor search uses the full action space of the robot $\mathcal{A}_0 = \bigcup_{r > 0} \mathcal{A}_r$ to construct the graph $G_0 = (V_0, E_0)$. This implies $V_0 = \bigcup_{r > 0} V_r$, and we refer to the anchor search vertex set as the \textit{union space}. If all coarse resolution states coincide with some state at the finest resolution ($V_r \subset V_1 \,\forall\, r > 1$), this is easily achieved by running the anchor search at the finest resolution ($V_0 \equiv V_1$).

Using multiple heuristics at the same resolution allows us to share information between these heuristics by maintaining a single \textit{cost-to-come} value (cost of the current best path between the start and some state) and multiple cost-to-goal estimates for a state. This information sharing allows the search to potentially escape local minima for some heuristic in a region of the state space on the basis of guidance from another heuristic at that resolution in that region.

%%%%%%%%%%%%%%%%%%%%%%%%%%%%%%%%%%%%%%%%%%%%%%%%%%%%%%%%%%%%%%%%%%%%%%%%%%%%%%%%

\section{Algorithm}\label{sec:algo}

\begin{algorithm}[tpbh]
\begin{small}
\caption{\small{\AMRA}}\label{alg:amra}
\begin{algorithmic}[1]

\Procedure{Key}{$x, i$}
    \Return $g(x) + w_1 \times h_i(x)$
\EndProcedure

\Procedure{Expand}{$x, i$}
    \State $r \gets \texttt{Res}(i)$
    \If{$i \neq 0$}
        \ForAll{$j > 0$}\label{line:open_clear} \Comment{Loop over search queues}
            \If{$j\neq i \wedge \texttt{Res}(j) == r$}
                % \State \Comment{compare queue resolutions}
                \State Remove $x$ from $OPEN_j$
            \EndIf
        \EndFor
    \EndIf
    \For{$x^\prime \in \texttt{Succs}(x, \mathcal{A}_r)$}\label{line:succs}
        \If{$g(x^\prime) > g(x)  + c(x, x^\prime)$}
            \State $g(x^\prime) \gets g(x)  + c(x, x^\prime)$
            \State $bp(x^\prime) \gets x$
            \If{$x^\prime \in CLOSED_0$}\label{line:incons}
                \State $\texttt{Insert}(x^\prime, INCONS)$
            \Else
                \State $\texttt{Update}(x^\prime, OPEN_0, \textsc{Key}(x^\prime, 0)$\label{line:succ_open}
                \ForAll{$j \in \{1, \ldots, M\}$}
                    \State $l \gets \texttt{Res}(j)$
                    \If{$l \notin \texttt{Resolutions}(x^\prime)$}
                        \State \textbf{continue}
                    \EndIf
                    \If{$x^\prime \notin CLOSED_l$}
                        \If{$\textsc{Key}(x^\prime, j) \leq w_2 \times \textsc{Key}(x^\prime, 0)$}
                            \State $\texttt{Update}(x^\prime, OPEN_j, \textsc{Key}(x^\prime, j))$\label{line:succ_inad}
                        \EndIf
                    \EndIf
                \EndFor
            \EndIf
        \EndIf
    \EndFor
\EndProcedure

\Procedure{ImprovePath}{ }
    \While{$OPEN_i$ is not empty $\forall i \in \{0, \ldots M\}$}\label{line:fail}
        % \State \Comment{Check for any $x_\text{goal} \in \mathcal{G} \subset \mathcal{X}$}
        \State $i \gets \texttt{ChooseQueue}()$\label{line:roundrobin} \Comment{over inadmissible queues}
        \If{$OPEN_i.\text{min}() \leq w_2 \times OPEN_0.\text{min}()$}
            \State $x \gets OPEN_i.\text{top}()$\label{line:expand_inad} \Comment{$OPEN_i.\text{top}()$ also pops}
            \State $\textsc{Expand}(x, i)$
            \State $r \gets \texttt{Res}(i)$
            \State $\texttt{Insert}(x, CLOSED_r)$\label{line:close_inad}
            \If{$x \in \mathcal{G}$}\label{line:goal_inad}
                \State $x_{\text{goal}} \gets x$
                \Return true\label{line:solve_inad}
            \EndIf
        \Else
            \State $x \gets OPEN_0.\text{top}()$\label{line:expand_anchor}
            \State $\textsc{Expand}(x, 0)$
            \State $\texttt{Insert}(x, CLOSED_0)$\label{line:close_anchor}
            \If{$x \in \mathcal{G}$}\label{line:goal_anchor}
                \State $x_{\text{goal}} \gets x$
                \Return true\label{line:solve_anchor}
            \EndIf
        \EndIf
    \EndWhile
\EndProcedure

\Procedure{Main}{$x_s, \mathcal{G}, \{\mathcal{A}_r\}_{r=0}^N, \{h_i\}_{i=0}^M, w_1^{\text{init}}, w_2^{\text{init}}$}
    \State $w_1 \gets w_1^{\text{init}}, w_2 \gets w_2^{\text{init}}$
    \State $g(x_s) = 0$
    \State $bp(x_s) \gets NULL$
    \ForAll{$i \in \{0, \ldots, M\}$}
        \State $OPEN_i.\text{clear}()$ \Comment{Initialise priority queues}
    \EndFor
    \State $\texttt{Insert}(x_s, INCONS)$
    \While{$w_1 \geq 1 \,\wedge\, w_2 \geq 1$}\label{line:amra_loop}
        \ForAll{$x \in INCONS$}\label{line:incons_to_anchor}
            \State $\texttt{Update}\left(x, OPEN_0, \textsc{Key}(x, 0)\right)$
        \EndFor
        \State $INCONS.\text{clear}()$
        \ForAll{$x \in OPEN_0$}\label{line:anchor_to_inad}
            \ForAll{$j \in \{1, \ldots, M\}$}
                \If{$\texttt{Res}(j) \in \texttt{Resolutions}(x)$}
                    \State $\texttt{Update}\left(x, OPEN_j, \textsc{Key}(x, j)\right)$
                \EndIf
            \EndFor
        \EndFor
        \ForAll{$r \in \{0, \ldots, N\}$}
            \State $CLOSED_r.\text{clear}()$
        \EndFor
        \If{$\textsc{ImprovePath}( )$}
            \State Publish current solution by tracing $bp(x_{\text{goal}})$ till $x_s$
        \EndIf
        \If{$w_1 == 1 \,\wedge\, w_2 == 1$}
            \State \textbf{break}
        \EndIf
        \State Update $w_1, w_2$ \label{line:w_update}
    \EndWhile
\EndProcedure
\end{algorithmic}
\end{small}
\end{algorithm}

Algorithm~\ref{alg:amra} contains the full \AMRA search procedure. \AMRA is initialised with the start state $x_s$, goal set $\mathcal{G}$, action spaces $\{\mathcal{A}_0, \ldots, \mathcal{A}_N\}$, and heuristics $\{h_0, \ldots, h_M\}$. The state space $\mathcal{X}$ is discretised into $N$ levels. For $i < j$, resolution $i$ is finer than resolution $j$. The anchor action space is the union action space ($\mathcal{A}_0 = \bigcup_{r=1}^N \mathcal{A}_r$), and the anchor search is run at the finest resolution ($V_0 \equiv V_1$). The anchor heuristic $h_0$ is consistent (and thus admissible), while the other heuristics may be inadmissible. There is at least one heuristic per resolution, thus $M \geq N$.

\subsection{Connections to Existing Algorithms}
\AMRA is a generalisation of several existing search algorithms. With a single heuristic per resolution, if we do not run \AMRA anytime, \AMRA is the same as \MRA~\cite{DuIL20}. We can also run \AMRA for a single resolution, with multiple heuristics at that resolution, and obtain either \AMHA~\cite{NatarajanSALC19} or \MHA~\cite{AineSNHL16} depending on whether it is run anytime or not. In slightly more contrived scenarios, for a single resolution and a single heuristic, \AMRA can also devolve to \ARA~\cite{LikhachevGT03} and Weighted A* (\wA)~\cite{Pohl70}. The connections stem from the fact that \AMRA utilises multiple resolutions, multiple heuristics, and is anytime.

\subsection{\AMRA Desiderata}

We denote the cost-to-come for a state with the function $g: V \rightarrow \mathbb{R}_{\geq 0}$. The \textit{parent} of a state $x$, denoted by $bp(x)$ is its predecessor on the best known path from $x_s$ to $x$. \texttt{Resolutions}$(x)$ returns the set of resolutions state $x$ lies on: $r \in \texttt{Resolutions}(x) \Rightarrow x \in V_r$. Each resolution $r$ is associated with a container for states expanded at that resolution, $CLOSED_r$. \texttt{Res}$(i)$ returns the resolution associated with heuristic $h_i$. Each heuristic $h_i$ is associated with a priority queue $OPEN_i$. \texttt{Succs}$(x, \mathcal{A}_r)$ generates all valid successors of $x$ at resolution $r$ using the appropriate action space $\mathcal{A}_r$. For $r = 0$, this generates all valid successors of $x$ for all resolutions in \texttt{Resolutions}$(x)$.

\subsection{Algorithmic Details}
The anytime nature of \AMRA is controlled by the loop in Line~\ref{line:amra_loop}. The suboptimality of the solution is controlled by parameters $w_1, w_2$.
At the end of each iteration, \AMRA returns a solution which is at most $w_1\times w_2$ suboptimal with respect to the graph $G_0 = (V_0, E_0)$ (from Theorem~\ref{thm:suboptimality}). $w_1, w_2$ are decreased in Line~\ref{line:w_update} in order to potentially improve the solution quality in the next iteration. To facilitate this, \AMRA maintains $INCONS$ - a container for all \textit{inconsistent} states. These are states whose cost-to-come, or $g$-value, is improved after they have been expanded from the admissible anchor search. If a state becomes inconsistent, a better solution through it might be found than the current best known solution. Hence these states are added back into the appropriate $OPEN_i$ for consideration by the search (Lines~\ref{line:incons_to_anchor} and~\ref{line:anchor_to_inad}).

\textsc{ImprovePath} is the core function that searches for a path between $x_s$ and $\mathcal{G}$. $x_{\text{goal}} \in \mathcal{G}$ is some state which satisfies the termination condition in Line~\ref{line:goal_inad} or Line~\ref{line:goal_anchor}. If no such $x_{\text{goal}}$ is found before all $OPEN_i$ are exhausted (Line~\ref{line:fail}), \AMRA terminates with failure. Line~\ref{line:roundrobin} controls the scheduling policy over all heuristics. While many options exist~\cite{PhillipsNAL15}, for \AMRA we use a simple round robin.
% If $g(x_{\text{goal}}) \leq w_2 \times OPEN_0.\text{min}()$, we have found a solution which is at most $w_2$ more suboptimal than any solution found by the anchor search. Since the anchor itself is equivalent to a \wA search with suboptimality factor $w_1$, an \AMRA solution is at most $w_1 \times w_2$ suboptimal.

The core modification in \AMRA over \MRA and \MHA is the way in which the graph is constructed in \textsc{Expand}. Any time a state is expanded at a particular resolution, it is removed from all inadmissible (non-anchor) heuristics at that resolution (in the loop in Line~\ref{line:open_clear})\footnote{For the sake of simplicity, we refer to all non-anchor heuristics as `inadmissible'.}. This is because the $g$-value of an inadmissibly expanded state is independent of the heuristic it was expanded from. The \textsc{Expand} function generates the successors of state $x$ by using the appropriate action space $\mathcal{A}_r$ (Line~\ref{line:succs}). After checking for successor consistency in Line~\ref{line:incons}, a newly generated state is inserted at all appropriate resolutions (Lines~\ref{line:succ_open} to~\ref{line:succ_inad}).

% At a high-level, during each search iteration, it can help to think of \AMRA as running one Weighted A* search per resolution. If we consider running completely independent Weighted A* searches with action spaces $\{\mathcal{A}_1, \ldots, \mathcal{A}_N\}$, and the same suboptimality factor $w$, we might get up to $N$ solution paths for a problem. The solution returned by \AMRA for the same problem (with $w_1 = w, w_2 = 1$ and the same $M = N$ heuristics) will be composed of partial paths from some of these $N$ solutions.

% If we initialise \AMRA with a single resolution, multiple heuristics at that resolution, and decide not to update $w_1, w_2$ in Line~\ref{line:w_update} (instead exiting the loop), we will get the same solution as \MHA for the problem. If we do update $w_1, w_2$, we would get the same set of solutions as the anytime version of \MHA \textcolor{red}{[?]} (provided the update schedules are the same).

% In the case where we do not update $w_1, w_2$, we could run \AMRA with multiple resolutions, a single heuristic per resolution, and obtain the same solution as \MRA.

% In this way, AMRA with a single resulution, multiple heuristics not anytimeis MHA and anytime is AMHA. WIth multiple resolutions and x heuristics per resolution not anytime, AMRA devolves to solutions from MRA

%%%%%%%%%%%%%%%%%%%%%%%%%%%%%%%%%%%%%%%%%%%%%%%%%%%%%%%%%%%%%%%%%%%%%%%%%%%%%%%%

\section{Theoretical Analysis}\label{sec:analysis}

\begin{theorem}\label{thm:expansions}
    \AMRA expands each state at most $N+1$ times per iteration.
\end{theorem}
\begin{proof}
    Any state that is expanded must be in some $OPEN_i$ (Lines~\ref{line:expand_inad} and~\ref{line:expand_anchor}). Upon admissible expansion from the anchor search, the state is inserted into $CLOSED_0$ (Line~\ref{line:close_anchor}) and never inserted into $OPEN_0$ again (Line~\ref{line:incons}). For inadmissible expansions, the state is removed from all $OPEN_i$ for the appropriate resolution $r$ in the loop in Line~\ref{line:open_clear}, and inserted into $CLOSED_r$ in Line~\ref{line:close_inad}. This can happen once per resolution. Thus a state can be expanded at most $N + 1$ times per iteration of \AMRA.\hfill
\end{proof}

\begin{theorem}\label{thm:complete}
    \AMRA is complete with respect to the graph $G_0 = (V_0, E_0)$.
\end{theorem}
\begin{proof}
    \AMRA can either terminate after finding a solution in Line~\ref{line:solve_inad} or~\ref{line:solve_anchor}, or without a solution after exhausting all $OPEN_i$ and exiting the loop in Line~\ref{line:fail}. Since $\mathcal{A}_0 = \bigcup_{r > 0}\mathcal{A}_r$ and any edge $e \in E_0$ is an action $a \in \mathcal{A}_0$, $V_0 = \bigcup_{r > 0} V_r$. A consequence of this is that any solution at any resolution $r \geq 0$ must exist in $G_0$. Furthermore, if \AMRA exits the loop in Line~\ref{line:fail}, no states in $V_0$ remain to be expanded. Thus, \AMRA terminates in failure \textit{iff} there is no solution in the graph $G_0$.\hfill
\end{proof}

\begin{theorem}\label{thm:suboptimality}
    At the end of each iteration \AMRA returns a solution, if one exists, that is at most $w_1 \times w_2$ suboptimal with respect to the optimal solution in graph $G_0 = (V_0, E_0)$.
\end{theorem}
\begin{proof}
    (Sketch) \AMRA is complete with respect to $G_0$ (from Theorem~\ref{thm:complete}). The anchor search is a \wA search with a consistent heuristic and suboptimality factor $w_1$. Thus if \AMRA terminates via the anchor search in Line~\ref{line:solve_anchor}, $g(x_{\text{goal}}) \leq w_1 \times g^*(x_{\text{goal}})$ (from~\cite{ARAFormal}). If \AMRA terminates via an inadmissible heuristic in Line~\ref{line:solve_inad}, $g(x_{\text{goal}}) \leq w_2 \times OPEN_0.\text{min}() \leq w_2 \times w_1 \times g^*(x_{\text{goal}})$ (from~\cite{AineSNHL16}). Thus any solution returned by \AMRA is at most $w_1 \times w_2$ suboptimal with respect to $G_0$.\hfill
\end{proof}

%%%%%%%%%%%%%%%%%%%%%%%%%%%%%%%%%%%%%%%%%%%%%%%%%%%%%%%%%%%%%%%%%%%%%%%%%%%%%%%%

\section{Experimental Results}\label{sec:exps}

\subsection{Illustrative Example}

\begin{figure}[t]
    \centering
    \includegraphics[width=0.9\columnwidth]{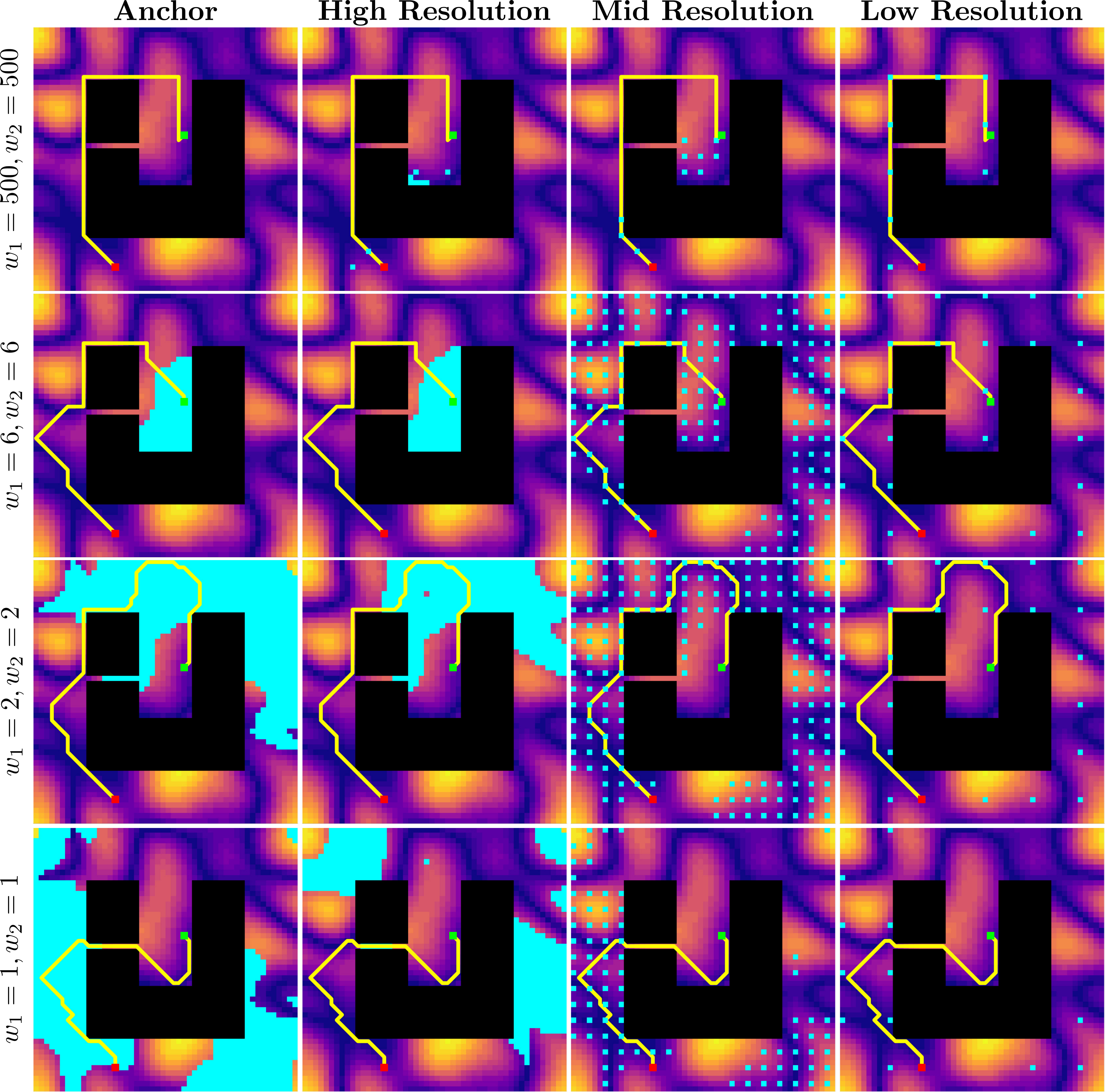}
    \caption{\AMRA execution on a 2D grid map with non-uniform costs. Start state (inside cul-de-sac) is in green, goal state in red, states expanded in cyan, and solution path in yellow. Each row is one iteration within \AMRA, and each column shows expansions from different state space discretisations. Cell costs increase from purple to orange. Best viewed in colour.}
    \label{fig:amra_ex}
\end{figure}

Fig.~\ref{fig:amra_ex} shows a 2D grid navigation example to illustrate the behaviour showed by \AMRA. We run \AMRA on a $50 \times 50$ map with three levels of discretisation: high ($1 \times 1$), mid ($3 \times 3$), and low ($9 \times 9$). Each $1 \times 1$ cell in the map has an assigned cost in the range $[10, 260]$. The robot can execute actions on an 8-connected grid at all resolutions, and the cost of an action is the sum of costs of $1 \times 1$ cells along that action. Only a single Euclidean distance heuristic was used. After finding an initial solution mostly at the low resolution, \AMRA expands more states at the finer resolutions over subsequent iterations to improve solution quality and finally terminates with the optimal solution for $w_1 = 1, w_2 = 1$.

\subsection{2D Grid Navigation}\label{sec:2dexps}

\begin{table*}[t]
\centering
\caption{2D Grid Navigation Results}
\label{tab:2dexps}
\begingroup
\setlength{\tabcolsep}{4pt}
% \scriptsize
\begin{tabular}{l||*{2}c|*{2}c|*{2}c|*{2}c|*{2}c|*{2}c}
\toprule
% & \multicolumn{12}{c}{\footnotesize\textbf{Planning Algorithms}} \\
% \cmidrule{2-13}
& \multicolumn{2}{c}{\textbf{\AMRA}} & \multicolumn{2}{c}{\textbf{\MRA}} & \multicolumn{2}{c}{\textbf{\ARA} (High)} & \multicolumn{2}{c}{\textbf{\ARA} (Mid)} & \multicolumn{2}{c}{\textbf{\ARA} (Low)} & \multicolumn{2}{c}{\textbf{\RRT}} \\
% \midrule
\cmidrule(lr){2-3}\cmidrule(lr){4-5}\cmidrule(lr){6-7}\cmidrule(lr){8-9}\cmidrule(lr){10-11}\cmidrule(lr){12-13}
\textbf{Metrics} & Cauldron (M1) & TheFrozenSea (M2) & M1 & M2 & M1 & M2 & M1 & M2 & M1 & M2 & M1 & M2  \\ \hline\hline
Success $\%$ & 100 & 100 & 100 & 100 & 100 & 100 & 98 & 99 & 20 & 30 & 100 & 100\\ \hline
$T_i (\si{\milli\second})$ & 1.2 $\pm$ 0.98 & 1.08 $\pm$ 1.04 & 1.03$\times$ & 1.01$\times$ & 11.27$\times$ & 13.63$\times$ & 0.73$\mathbf\times$ & 0.49$\times$ & 0.27$\times$ & 0.32$\mathbf\times$ & 158.84$\times$ & 256.85$\times$\\ \hline
$T_f (\si{\milli\second})$ & 280.39 $\pm$ 302.88 & 223.81 $\pm$ 280.36 & 1.4$\times$ & 1.37$\times$ & 1.13$\times$ & 1.46$\times$ & 0.07$\mathbf\times$ & 0.04$\mathbf\times$ & 0.07$\times$ & 0.05$\times$ & 53.88$\times$ & 30.36$\times$\\ \hline
$c_i$ & 1163.06 $\pm$ 574.31 & 953.14 $\pm$ 480.8 & 1$\times$ & 1$\times$ & 1.02$\times$ & 1.01$\times$ & 1$\times$ & 0.98$\times$ & 1.19$\times$ & 1.38$\times$ & 0.74$\mathbf\times$ & 0.81$\mathbf\times$\\ \hline
$c_f$ & 902.12 $\pm$ 403.27 & 754.7 $\pm$ 367.75 & 1$\times$ & 1$\times$ & 1$\times$ & 1$\times$ & 1.06$\times$ & 1.03$\times$ & 1.32$\times$ & 1.61$\times$ & 0.86$\mathbf\times$ & 0.85$\mathbf\times$\\ \hline
$\lvert V_0 \lvert \,(\times 10^4)$ & 11.87 $\pm$ 11.63 & 9.51 $\pm$ 11.44 & 1.57$\times$ & 1.5$\times$ & 2.17$\times$ & 2.66$\times$ & 0.1$\mathbf\times$ & 0.1$\mathbf\times$ & 0.12$\times$ & 0.1$\mathbf\times$ & 6.1$\times$ & 1.71$\times$\\ \hline
\bottomrule
\end{tabular}
\endgroup
\end{table*}

We test the performance of \AMRA for a 2D grid navigation task on two $1024\times1024$ maps from the MovingAI benchmark~\cite{sturtevant2012benchmarks} shown in Fig.~\ref{fig:2dmaps}. The state space was discretised at three levels: high ($1 \times 1$), mid ($7 \times 7$), and low ($21 \times 21$). A four-connected action space was used at each resolution and a single Manhattan distance heuristic was used. For each map, 100 random start and goal states were sampled at the low resolution. In this experiment, we compare the multi-resolution and anytime behaviour of \AMRA against \MRA and also \ARA run at each of the three resolutions denoted as ``\ARA (High)'', ``\ARA (Mid)'' and ``\ARA (Low)''. Since \MRA is not anytime, we ran a succession of \MRA searches with the same schedule of suboptimality weights as \AMRA. Additionally, we compare against an asymptotically optimal sampling-based planner \RRT~\cite{KaramanF11} from OMPL~\cite{sucan2012the-open-motion-planning-library}.

Table~\ref{tab:2dexps} presents the result of these experiments. We report six metrics: success rate, times to initial and final solutions ($T_i$ and $T_f$ respectively, in milliseconds), costs of initial and final solutions ($c_i$ and $c_f$ respectively), and the number of state expansions $\lvert V_0 \lvert$\footnote{For \RRT, $\lvert V_0 \lvert$ is the number of vertices in the final tree.}. All planners were given a $5 \si{\second}$ timeout. We report raw numbers for \AMRA and relative numbers for the other algorithms, averaged over the 100 trials.

\AMRA is faster than the complete search-based baselines (\MRA and \ARA (High)) and expands fewer states, while converging to the optimal solution. The convergence behaviour of these algorithms is shown in Fig.~\ref{fig:converge}. \AMRA is also much faster than \RRT\footnote{The termination criteria for \RRT was computing 10 solutions in a row whose costs were within $10\%$ of each other.}, albeit finding costlier solutions on a discretised grid representation of the environment.

\begin{figure}[t]
    \centering
    \includegraphics[width=0.7\columnwidth]{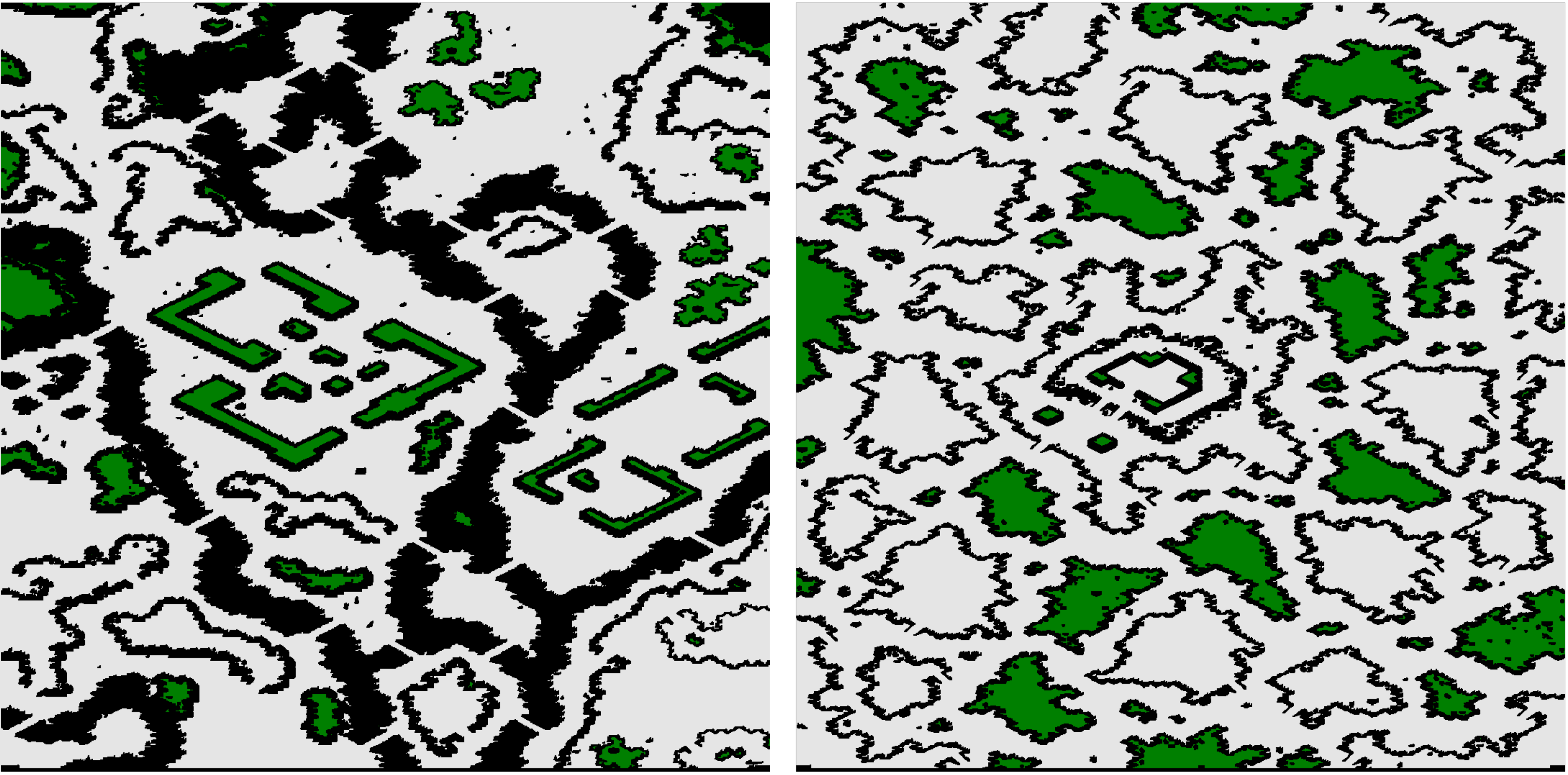}
    \caption{The four Starcraft maps from the MovingAI benchmark used for 2D grid navigation experiments: Cauldron (\textit{left}) and TheFrozenSea (\textit{right}). \textcolor{Green}{Green} and black areas are obstacles.}
    \label{fig:2dmaps}
\end{figure}

\begin{figure}[t]
    \centering
    \includegraphics[width=0.6\columnwidth]{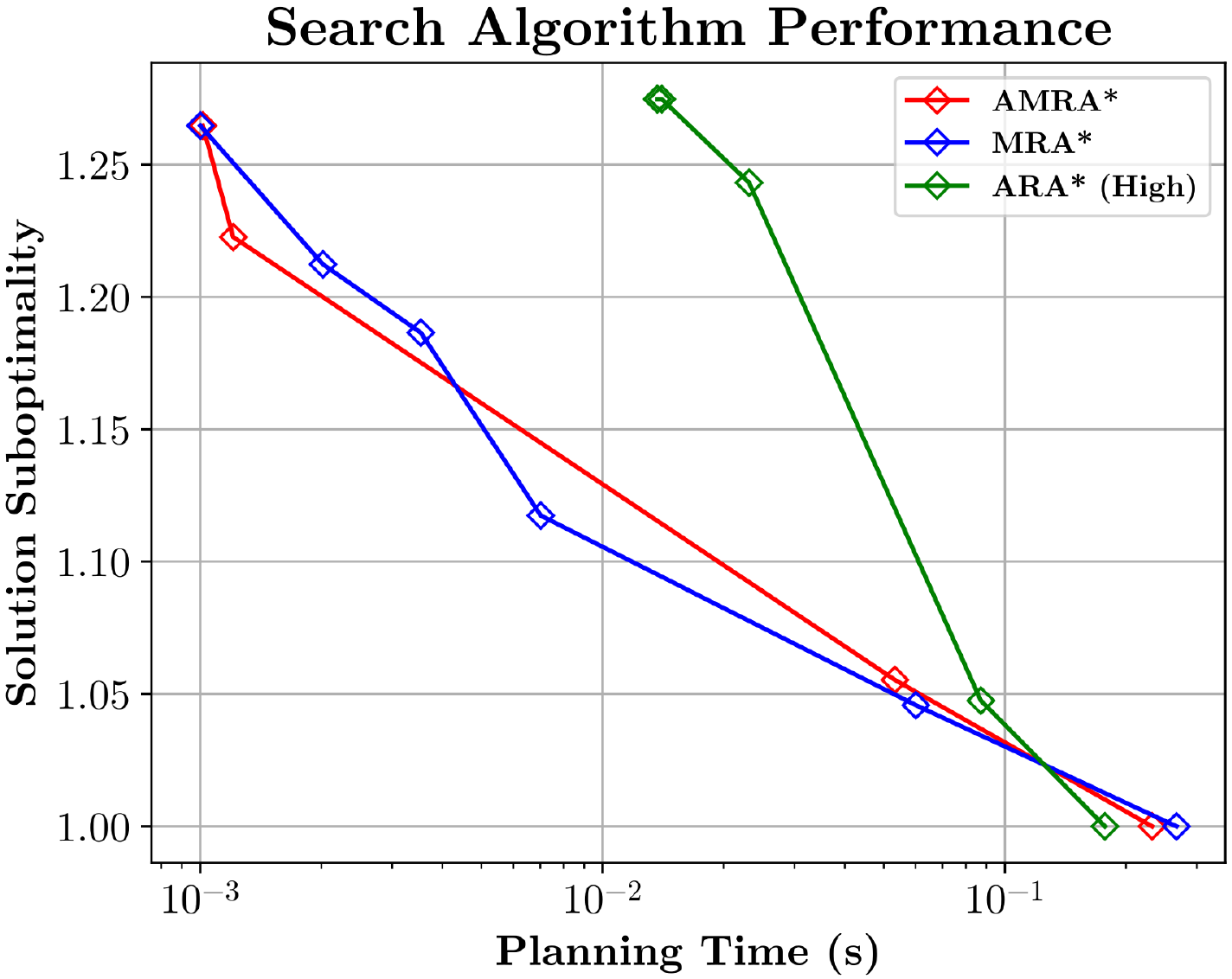}
    \caption{Performance of search algorithms on TheFrozenSea map from Fig.~\ref{fig:2dmaps}. Data was averaged over 100 runs. The x-axis is in log scale. \MRA was run iteratively as described in Sec.~\ref{sec:2dexps}.}
    \label{fig:converge}
\end{figure}

\subsection{UAV Navigation}\label{sec:uavexps}

\begin{table*}[t]
\centering
\caption{2D Grid Navigation Results}
\label{tab:uavexps}
\begingroup
\setlength{\tabcolsep}{6pt}
% \scriptsize
\begin{tabular}{l||*{2}c|*{2}c|*{2}c|*{2}c|*{2}c}
\toprule
% & \multicolumn{12}{c}{\footnotesize\textbf{Planning Algorithms}} \\
% \cmidrule{2-13}
& \multicolumn{2}{c}{\textbf{\AMRA}} & \multicolumn{2}{c}{\textbf{\MRA} (E)} & \multicolumn{2}{c}{\textbf{\MRA} (Dubins)} & \multicolumn{2}{c}{\textbf{\MRA} (Dijkstra)} & \multicolumn{2}{c}{\textbf{\AMHA} (High)}\\
% \midrule
\cmidrule(lr){2-3}\cmidrule(lr){4-5}\cmidrule(lr){6-7}\cmidrule(lr){8-9}\cmidrule(lr){10-11}
\textbf{Metrics} & Boston (M1) & NewYork (M2) & M1 & M2 & M1 & M2 & M1 & M2 & M1 & M2  \\ \hline\hline
$T_i (\si{\second})$ & 0.31 $\pm$ 0.25 & 0.26 $\pm$ 0.23 & 0.69$\times$ & 1.4$\times$ & 1.11$\times$ & 0.33$\times$ & 1.01$\mathbf\times$ & 1.04$\times$ & 0.94$\times$ & 0.94$\mathbf\times$ \\ \hline
$T_f (\si{\second})$ & 10.46 $\pm$ 10.77 & 9.54 $\pm$ 10.65 & 1.9$\times$ & 1.97$\times$ & 3.4$\times$ & 3.46$\times$ & 0.74$\mathbf\times$ & 0.75$\mathbf\times$ & 12.32$\times$ & 19.89$\times$ \\ \hline
$c_i$ & 135.64 $\pm$ 66.15 & 109.48 $\pm$ 49.52 & 1.12$\times$ & 1.06$\times$ & 1.36$\times$ & 1.37$\times$ & 1.04$\times$ & 0.98$\times$ & 1.86$\times$ & 2.07$\times$ \\ \hline
$c_f$ & 105.34 $\pm$ 46.61 & 92.47 $\pm$ 39.46 & 1$\times$ & 1$\times$ & 1$\times$ & 1$\times$ & 1$\times$ & 1$\times$ & 2.23$\times$ & 2.35$\times$ \\ \hline
$\lvert V_0 \lvert \,(\times 10^5)$ & 1.74 $\pm$ 1.9 & 1.56 $\pm$ 1.87 & 1.27$\times$ & 1.4$\times$ & 4.84$\times$ & 4.15$\times$ & 0.82$\mathbf\times$ & 0.83$\mathbf\times$ & 20.59$\times$ & 42.99$\mathbf\times$ \\ \hline
Timeout $\%$ & 15 & 12 & 39 & 29 & 30 & 23 & 6 & 4 & 74 & 82 \\ \hline
\bottomrule
\end{tabular}
\endgroup
\end{table*}

The second set of experiments studies the multi-heuristic capabilities of \AMRA in addition to the multi-resolution and anytime behaviour. We solve kinodynamic motion planning problems for a 4D UAV robot modeled with double integrator dynamics. The state space of the robot is $(x, y, \theta, v)$ - its 2D pose in $SE(2)$ and linear velocity. The motion primitives used for the search algorithms are shown in Fig.~\ref{fig:mprims}. They exist at two resolutions for the 2D position: $3\si{\meter}$ (high) and $9\si{\meter}$ (low). The heading $\theta$ can take 12 discrete values in $[0, 2\pi)$, and the velocity $v$ can be $\{0, 3, 8\} \si{\meter\per\second}$ at the end of a primitive. The cost of an action is its duration, thus we are solving for the least-time path in this experiment.

\begin{figure}[t]
    \centering
    \includegraphics[width=0.55\columnwidth]{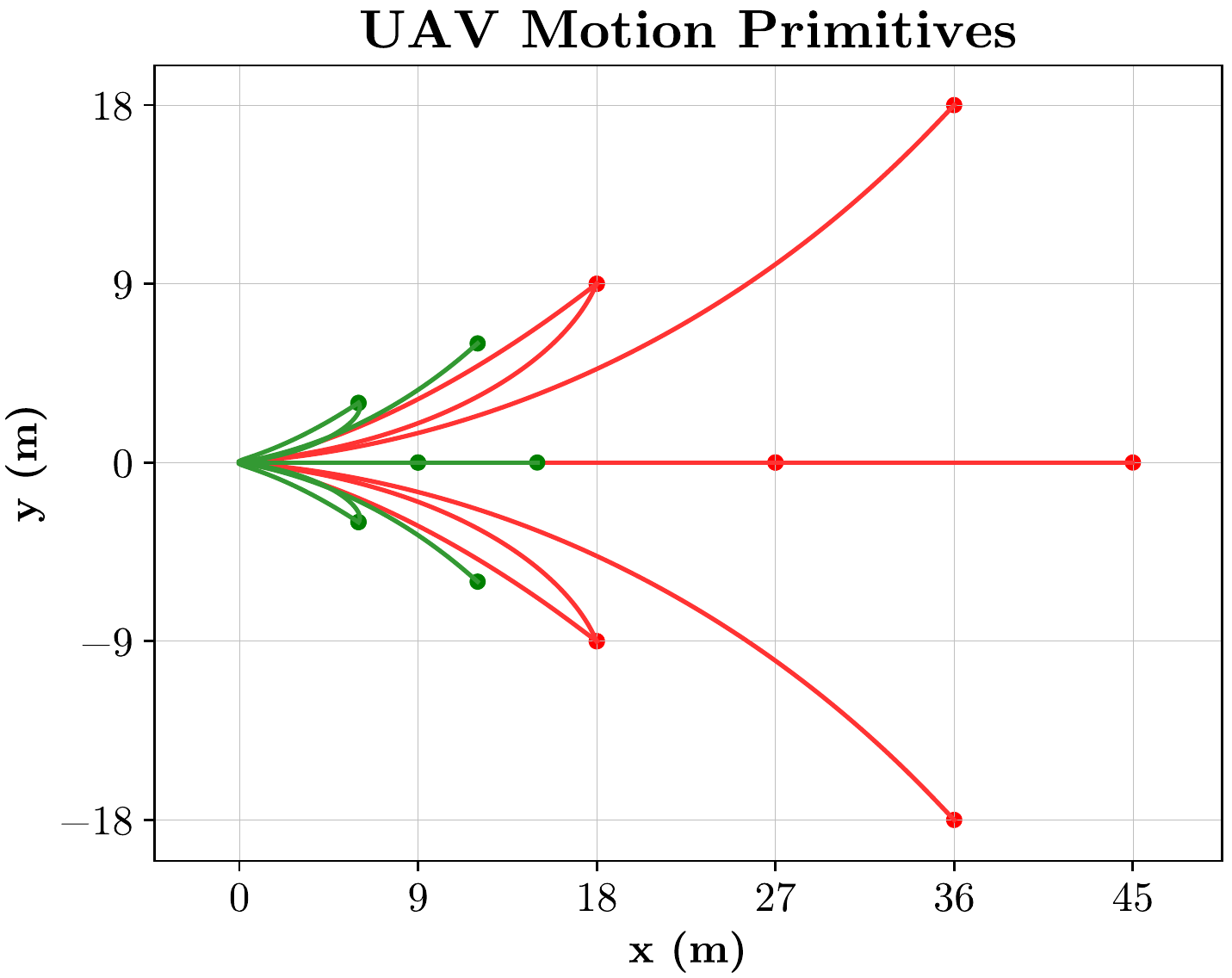}
    \caption{Motion primitives for UAV navigation. High resolution primitives are in \textcolor{Green}{green}, and low resolution primitives are in \textcolor{Red}{red}.}
    \label{fig:mprims}
\end{figure}

We use three inadmissible heuristics for this experiment: Euclidean distance to the goal (always used as the admissible anchor heuristic as well), Dubins path~\cite{Dubins} distance to the goal, and a backwards Dijkstra search from the goal. \AMRA uses all three heuristics at both resolutions. 100 random start and goal states were sampled in two maps shown in Fig.~\ref{fig:uavmaps} at the low resolution, and planners were given a timeout of $30 \si{\second}$. We compare against two search-based algorithms: \MRA and \AMHA. The former is not a multi-heuristic algorithm, thus we compare against instantiations which use different heuristics: ``\MRA (E)'', ``\MRA (Dubins)'', and ``\MRA (Dijkstra)''. We also compare against ``\AMHA (High)'' since that is not a multi-resolution algorithm. ``\AMHA (Low)'' was unable to find any solutions across the $2\times100$ problems.

Table~\ref{tab:uavexps} shows the results of these experiments. As in Section~\ref{sec:2dexps}, we present raw numbers for \AMRA and relative numbers for the other baselines, averaged over 100 trials. Since all these algorithms succeeded in finding an initial solution, we report the timeout percentage (percentage of problems that reached the planning timeout before finding the final solution) in place of success rate. Overall, \AMRA is the most consistent algorithm when compared against the baselines. It finds the optimal solution much faster than ``\MRA (E)'', ``\MRA (Dubins)'', and ``\AMHA (High)'' and with fewer timeouts. In most cases it is also quicker to find the first solution than all \MRA variants. ``\MRA (Dijkstra)'' is the most competitive baseline as it finds the optimal solution quicker, with fewer expansions and fewer timeouts. This comparison shows the effect of the overhead of \AMRA using multiple heuristics and multiple resolutions.
% With multiple resolutions and multiple heuristics, \AMRA can sometimes pay the price of expanding extra states that the baseline algorithms do not.

\begin{figure}[t]
    \centering
    \includegraphics[width=0.7\columnwidth]{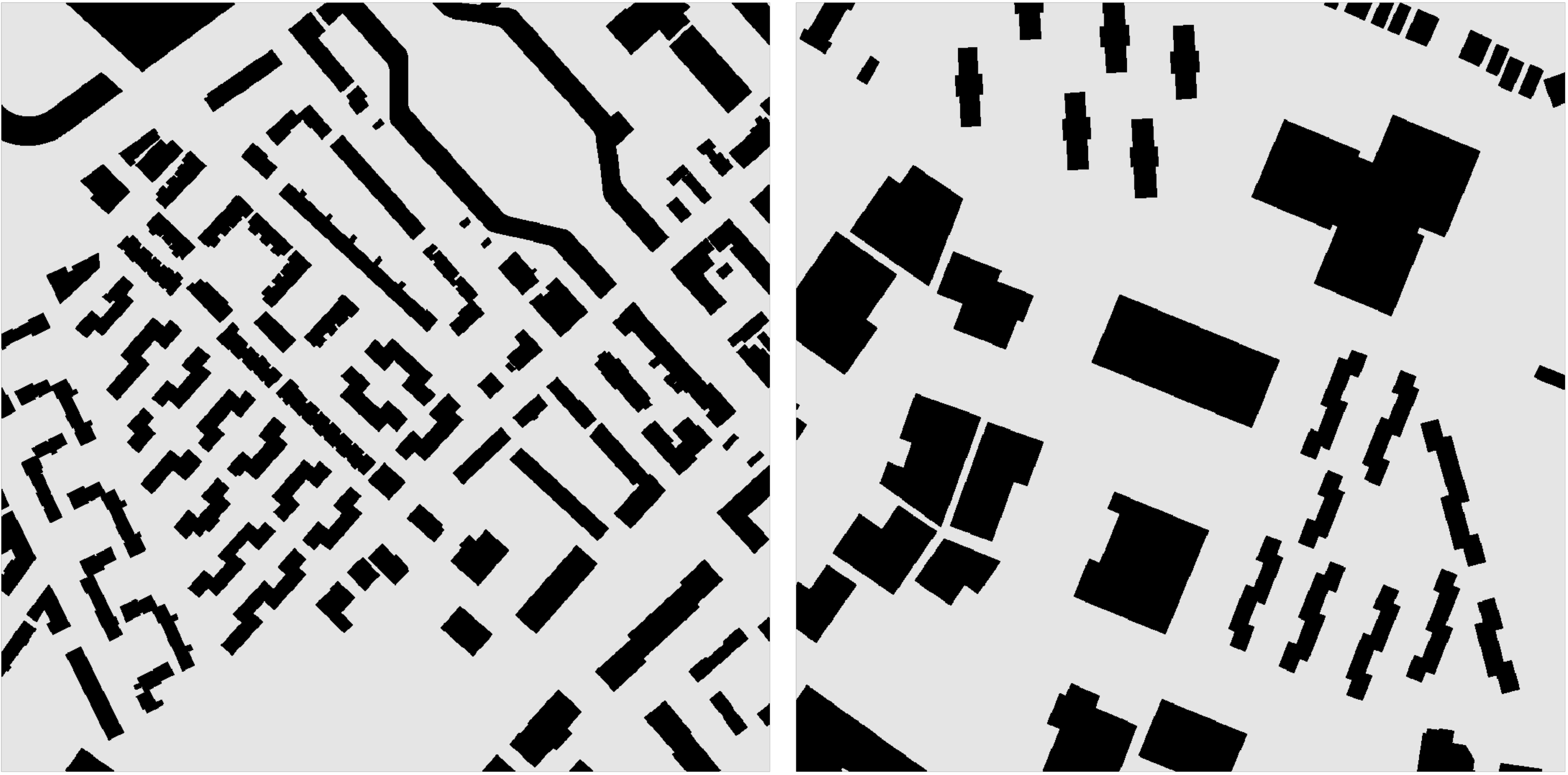}
    \caption{The two city/street maps from the MovingAI benchmark used for UAV navigation experiments: Boston\_0\_1024 (\textit{left}) and NewYork\_0\_1024 \textit{right}. Given the fixed discretisations for $\theta$ and $v$, there are roughly $2.8 \times 10^7$ valid states in these maps.}
    \label{fig:uavmaps}
\end{figure}

% \begin{table}[]
% \caption{Quantitative Performance for Real-World Experiments}
% \label{tab:realworld}
% \begingroup
% \setlength{\tabcolsep}{4pt}
% \begin{tabular}{lccc}
% \toprule
% & \multicolumn{3}{c}{\textbf{Metrics}} \\
% \cmidrule{2-4}
% \textbf{Algorithm} & Success Rate & Planning Time (\si{\second}) & Simulation Time (\si{\second}) \\ \midrule
% \textsc{SPAMP} & 82$\%$ & 2 $\pm$ 4 & 0.8 $\pm$ 0.6 \\ \bottomrule
% \end{tabular}
% \endgroup
% \end{table}

%%%%%%%%%%%%%%%%%%%%%%%%%%%%%%%%%%%%%%%%%%%%%%%%%%%%%%%%%%%%%%%%%%%%%%%%%%%%%%%%

\section{Discussion \& Future Work}\label{sec:future}
In this work we present \AMRA an anytime, multi-resolution, multi-heuristic search algorithm that generalises several existing search algorithms into one unified algorithm. It it very flexible for robot motion planning problems that have previously benefited from anytime algorithms, multiple heuristics, and multiple resolutions in separate lines of research. \AMRA exhibits impressive performance on two very different planning domains in 2D grid navigation and 4D kinodynamic UAV planning.

\AMRA at its core utilises multiple action spaces. Plenty of robotic systems are capable of a diverse set of actions that may dynamically become available to the robot given the state it is in. For example, a robot arm might plan in free space with simple motor primitives (independent joint angle changes), but might need to resort to prehensile and non-prehensile interaction actions in the vicinity of clutter. \AMRA opens the door for developing search algorithms that reason about such dynamically evolving action spaces that include both robot-centric and object-centric actions.

%%%%%%%%%%%%%%%%%%%%%%%%%%%%%%%%%%%%%%%%%%%%%%%%%%%%%%%%%%%%%%%%%%%%%%%%%%%%%%%

\balance
\bibliographystyle{IEEEtran}
\bibliography{references}

\end{document}